\documentclass{article}

\usepackage{microtype}
\usepackage{graphicx}
\usepackage{subfigure}
\usepackage{booktabs} 

\usepackage{hyperref}

\usepackage{amsmath}
\usepackage{amssymb}
\usepackage{mathtools}
\usepackage{amsthm}

\usepackage[capitalize,noabbrev]{cleveref}

\theoremstyle{plain}
\newtheorem{theorem}{Theorem}[section]
\newtheorem{proposition}[theorem]{Proposition}

\theoremstyle{definition}

\theoremstyle{remark}
\newtheorem{remark}[theorem]{Remark}
\theoremstyle{example}
\newtheorem{example}[theorem]{Example}

\usepackage[textsize=tiny]{todonotes}

\newcommand{\RR}{\mathbb{R}}

\newcommand{\rank}{\mathrm{rank}}
\newcommand{\softmax}{\mathrm{softmax}}
\newcommand{\argmax}{\mathrm{argmax}}

\renewcommand{\exp}{\mathrm{exp}}
\newcommand{\acc}{\mathrm{acc}}

\newcommand{\pto}{\nrightarrow}

\newcommand{\D}{\mathcal{D}}
\renewcommand{\L}{\mathcal{L}}
\newcommand{\T}{\mathcal{T}}

\newcommand{\K}{\mathcal{K}}
\newcommand{\Q}{\mathcal{Q}}
\newcommand{\V}{\mathcal{V}}

\usepackage{pict2e}
\DeclareRobustCommand{\mask}{%
  \begingroup
  \setlength{\unitlength}{1.5ex}%
  \begin{picture}(1,1)
  \polyline(0,1)(0,0)(1,0)(0,1)
  \end{picture}%
  \endgroup
}

\title{Paying Attention to Facts: Quantifying the Knowledge Capacity of Attention Layers}

\author{Liang Ze Wong}

\begin{document}

\maketitle

\begin{abstract}
In this paper, we investigate the ability of single-layer attention-only transformers (i.e. attention layers) to memorize facts contained in databases from a linear-algebraic perspective.
We associate with each database a 3-tensor, propose the rank of this tensor as a measure of the size of the database, and provide bounds on the rank in terms of properties of the database.
We also define a 3-tensor corresponding to an attention layer, and empirically demonstrate the relationship between its rank and database rank on a dataset of toy models and random databases.
By highlighting the roles played by the value-output and query-key weights, and the effects of argmax and softmax on rank, our results shed light on the `additive motif' of factual recall in transformers, while also suggesting a way of increasing layer capacity without increasing the number of parameters.
\end{abstract}

\section{Introduction}
\label{sec:introduction}



Large language models (LLMs) based on transformers with attention \cite{vaswani2017attention} have proven successful at a large variety of tasks. 
While much research has focused on \emph{how well} LLMs accomplish tasks, less has been done on \emph{how} LLMs accomplish these tasks, and on the nature of the tasks themselves. 
In this paper, we take a step towards a better theoretical understanding of how a single attention-only layer can memorize facts, such as those contained in databases, RDF triple stores or knowledge graphs.

While our work is inspired by recent papers on factual recall that have sought to uncover the underlying mechanisms behind it \cite{chughtai2024summing,nichani2024understanding,yao2024knowledge,geva2020transformer} or to derive scaling laws for knowledge capacity \cite{lu2024scaling,allen2024physics}, it differs from prior research in 3 key ways: 
1) Unlike prior research which includes multi-layer perceptron (MLP) layers \cite{geva2020transformer,nichani2024understanding}, we focus solely on the role of attention layers.
2) We propose linear algebraic measures of database and attention layer size, in contrast to prior results on scaling laws which are phrased in terms of number of triples and number of model parameters \cite{lu2024scaling,allen2023physics}. 
3) Instead of probing trained LLMs and reverse-engineering weights to determine where and how facts are stored \cite{chughtai2024summing,nichani2024understanding}, we start from a theoretical perspective and build towards a mechanistic understanding of knowledge storage.

Our motivation for these choices are the observations in the literature that transformers, especially attention-only models, have `an enormous amount of linear structure' \cite{elhage2021mathematical}, and that the task of factual recall in transformers is solved additively \cite{chughtai2024summing}, a phenomena termed the `additive motif'.
These observations suggest that a linear algebraic approach to the question of factual recall in attention layers should yield interesting results.
One of the key invariants in linear algebra is the rank of a linear operator, but discussions about model rank in the literature are limited to mentions of model dimensions (e.g. the embedding dimensions and attention-head dimensions) and usually for the purpose of computing other measures of model size such as the number of parameters (e.g. \cite{lu2024scaling,allen2024physics}).
In-depth investigations into the properties of the rank of models or databases are conspicuously missing.

The objective of this paper is to address these gaps, and provide an alternative linear-algebraic angle to this active area of research.
Post-hoc interpretability research can sometimes feel like trying to look for a needle in a haystack, and we hope this work makes the task slightly easier by helping to reveal what the needle looks like.
This paper represents a small step in the direction of defining theoretically-grounded and empirically-supported measures of database and model size, and in highlighting the different roles played by the query-key and value-output matrices in factual recall.
We hope that this will shed additional light on the additive motif, as well as on the difficulties in interpreting models, even simple ones. 
As we shall see, even in our highly simplified setting, the role of the rank of databases and of attention layers is clouded by many complications and far from being completely resolved.
Nevertheless, we hope our work inspires future research in similar directions on developing a theoretical understanding of tasks and how language models accomplish them.

\subsection{Outline}
The paper is organized as follows:
in Section \ref{sec:databases_and_their_rank}, we associate with each database a $3$-tensor, propose the rank of this tensor as a linear-algebraic measure of the size of the database, and derive bounds for its rank in terms of database properties.
In Section \ref{sec:attention_layers_and_their_rank}, we similarly construct a $3$-tensor associated to an attention layer and provide bounds on its rank.
Our results suggest a method of increasing layer capacity without increasing the number of parameters.
We also investigate the effect of $\argmax$ and $\softmax$ on rank, and propose methods of assessing model accuracy that are less sensitive to these effects.
We reinforce our theoretical results with illustrative examples, as well as with experiments conducted on a dataset of toy databases and models in Section \ref{sec:experimental_results}.
Finally, in Section \ref{sec:contributions_limitations_future_work}, we summarize our findings and discuss limitations and extensions of our work.


\section{Databases and their rank}
\label{sec:databases_and_their_rank}
We will be interested in \emph{databases}, by which we mean collections of tables or dataframes like those in Fig.\ \ref{fig:db_countries}.
To each such database, we can associate a ``language'' consisting of Resource Description Framework (RDF) triples or ``sentences'' of the form ``[subject] [predicate] [object]''.
We are interested in encoding or storing the information from a database in single-layer attention-only transformer models. 
By `encoding or storing', we mean training a model on RDF triples from the database, so that the model can correctly complete partial triples of the form ``[subject] [predicate]''.
For example, a model trained on the sentences in Fig.\ \ref{fig:db_countries} should predict ``Singapore'' as the next token, when given the incomplete sentence ``Astrid born\_in''.

\begin{figure}[ht]
\vskip 0.2in
\begin{center}
	\centerline{	
	\begin{tabular}{| l || l | l |}
		\hline
		 	& born\_in & lives\_in
		\\
		\hline
		Astrid & Singapore & Malaysia
		\\
		Bernard & Singapore & Singapore
		\\
		Colin & Malaysia & Malaysia
		\\
		\hline
	\end{tabular}
	\begin{tabular}{| l || l |}
		\hline
		 	& currency
		\\
		\hline
		Malaysia & Ringgit
		\\
		Singapore & Dollar
		\\
		\hline
	\end{tabular}	
	}
	\centerline{
	\begin{tabular}{r l l}
		& & \\	
		Astrid & born\_in & Singapore \\ 
		Bernard & born\_in & Singapore \\ 
		Colin & born\_in & Malaysia \\ 
		Astrid & lives\_in & Malaysia \\ 
		Bernard & lives\_in & Singapore \\ 
		Colin & lives\_in & Malaysia \\
		Malaysia & currency & Ringgit \\
		Singapore & currency & Dollar
	\end{tabular}		
	}
\caption{A database with two tables, and its `language' of RDF triples in subject-predicate-object format.}
\label{fig:db_countries}
\end{center}
\vskip -0.2in
\end{figure}

			
				

This memorization task is a simplification of how larger transformer models store facts of the form, ``France's capital is Paris''.\footnote{An LLM trained on natural language would have to handle the contenation of separate tokens like `capital' and `is' into the predicate `capital is', but in this paper, we avoid such issues and simply assume that all subjects, predicates and objects are their own single tokens.}
We wish to understand which parts of a transformer are capable of storing such information and how the information is encoded in the weights of the model.

Abstractly, a database is a \emph{partial function} 
$$
\D \colon \T \times \T \pto \T,
$$
where $\T$ is the set of all tokens in $\D$, and $\pto$ denotes a function that is only defined on a subset of its domain.
Let $\K, \Q$ and $\V$ denote the subsets of $\T$ consisting of tokens that appear in $D$ as subjects, predicates, and objects, respectively.\footnote{Subjects, predicates and objects can be thought of as keys, queries and values, resp., hence the letters $\K, \Q$ and $\V$. Note that $\K, \Q$ and $\V$ need not be disjoint.} 
We can then treat $\D$ as a partial function $\K \times \Q \pto \V$.

We will also use $\D$ to refer to the set of tuples $\{ (k,q,v) | \D(k,q) = v \}$.
In addition statements of the form ``$(k,q,v) \in \D$'', we will say ``$(k,q) \in \D$'' if $(k,q)$ is in the domain of the function $\D$, i.e. $(k,q,v) \in \D$ for some $v$.

To any such database, we may associate a 3-tensor $D \in \RR^{|\K| \times |\Q| \times |\V|}$ where
\begin{equation}
\label{eq:def_D}
	D_{kqv} = \begin{cases}
1 \mbox{ if } (k,q,v) \in \D, \mbox{ and}\\
0 \mbox{ otherwise.}
\end{cases}
\end{equation}
\emph{Slices} of $D$ are matrices
$
	D_{k::} \in \RR^{|\Q| \times |\V|},  D_{:q:} \in \RR^{|\K| \times |\V|}, 
$ and 
$
	D_{::v} \in \RR^{|\K| \times |\Q|},  
$
obtained by fixing $k,q$ or $v$ and varying the other dimensions, while \emph{fibers} of $D$ are the vectors
$D_{kq:}$, $D_{k:v}$ and $D_{:qv}$ obtained by fixing a pair of $k,q$ or $v$.

We define the \emph{rank} of a database to be the rank of its associated tensor.\footnote{The rank of a $3$-tensor is the minimum number of simple tensors $\mathbf{u} \otimes \mathbf{v} \otimes \mathbf{w}$ that need to be summed up to compose it.}
Unlike the rank of a matrix (i.e. 2-tensors), which can easily be determined (e.g., via Gaussian elimination), there are no easy ways to determine the rank of an $n$-tensor when $n > 2$; indeed, even approximating tensor rank is NP-hard \cite{haastad1990tensor, swernofsky2018tensor}.
The number of triples in a database, which we denote $|\D|$, gives an upper bound to its rank. 
But we can get a tighter bound:
\begin{proposition}
\label{prop:db_rank}
Let $\mathcal{V}_k$ and $\mathcal{V}_q$ be the subsets of $\mathcal{V}$ that occur as possible objects in triples of $\D$ with subject $k$ or predicate $q$, respectively.
Then
\begin{align*}	
\rank(D) &\leq \min\left(\sum_{k \in \K} \rank(D_{k::}), \sum_{q \in \Q} \rank(D_{:q:})\right)
\\
	&= \min\left( \sum_{k \in \K} |\V_k|, \sum_{q \in \Q} |\V_q| \right)  \quad \leq |\D|.
\end{align*}
\end{proposition}
\begin{proof} 
The rank of a $3$-tensor is less than the rank of the sum of its slices, and the ranks of each $k$ (treated as a function $\Q \to \V$) and each $q$ (treated as a function $\K \to \V$) are the number of unique values in $\V$ that they take.
The inequality with $|\D|$ follows from $|\D| = \sum_{k \in \K} |\Q_k| = \sum_{q \in \Q} |\K_q|$ and $|\V_k| \leq |\Q_k|$ and $|\V_q| \leq |\K_q|$, where $\Q_k$ and $\K_q$ are defined analogously to $\V_k, \V_q$.
\end{proof}

\begin{example}
\label{ex:db}
Let $\D$ be the database in Fig.\ \ref{fig:db_countries}, with abbreviations $a = \mathrm{Astrid}, b=\mathrm{Bernard}, \dots, \beta =\mathrm{born\_in}, \lambda=\mathrm{lives\_in}$ and $\kappa=\mathrm{currency}$. 

Then we have $\K = \{a,b,c,m,s\}, \Q = \{\beta, \lambda, \kappa\}$ and $\V = \{m, s, r, d\}$. Examples of slices are:
$$
D_{a::} = 
\begin{array}{cc}
    &
    \begin{array}{cccc}
    m & s & r & d 
    \end{array}
    \\
    \begin{array}{c}
    \beta \\ \lambda \\ \kappa
    \end{array}
    &
    \left(
    \begin{array}{cccc}
    0 & 1 & 0 & 0 
    \\
    1 & 0 & 0 & 0 
    \\
    0 & 0 & 0 & 0 
    \\
    \end{array}
    \right)
\end{array}
$$
$$
D_{:\beta:} = 
\begin{array}{cc}
    &
    \begin{array}{cccc}
    m & s & r & d 
    \end{array}
    \\
    \begin{array}{c}
    a \\ b \\ c \\ m \\ s
    \end{array}
    &
    \left(
    \begin{array}{cccc}
    0 & 1 & 0 & 0 
    \\
    0 & 1 & 0 & 0 
    \\
    1 & 0 & 0 & 0 
    \\
    0 & 0 & 0 & 0 
    \\
    0 & 0 & 0 & 0 
    \end{array}
    \right)
\end{array}
$$
$$
D_{:\lambda:} = 
\begin{array}{cc}
    &
    \begin{array}{cccc}
    m & s & r & d 
    \end{array}
    \\
    \begin{array}{c}
    a \\ b \\ c \\ m \\ s
    \end{array}
    &
    \left(
    \begin{array}{cccc}
    1 & 0 & 0 & 0 
    \\
    0 & 1 & 0 & 0 
    \\
    1 & 0 & 0 & 0 
    \\
    0 & 0 & 0 & 0 
    \\
    0 & 0 & 0 & 0 
    \end{array}
    \right)
\end{array}
$$
From the proposition, we have $\rank(D) \leq 6$, while $|\D| = 8$.
In fact, $\rank(D) = 5$, due to $\beta$ and $\lambda$ taking the same values on $b$ and $c$.
\end{example}

\begin{remark}
As the example shows, the rank of a $3$-tensor can be less than the sum of the ranks of its slices.
Intuitively, $|\D|$ counts facts assuming they are all independent, the tighter bounds from Proposition \ref{prop:db_rank} account for patterns \emph{within} slices but assumes the slices are independent, while the actual rank accounts for patterns \emph{across} slices.
We shall see that attention layers can exploit these patterns across slices.
\end{remark}

\subsection{Low-rank representations of databases}
\label{sec:low_rank_representations_of_databases}
Before moving on to the rank of attention layers, we pause to point out a peculiarity in how factual recall is usually evaluated.
A standard practice when assessing the factual recall of models is to train models on all triples $(k, q, v)$ in a database (e.g. the `sentences' at the bottom of Figure \ref{fig:db_countries}), and evaluate them by passing only the initial pairs $(k,q)$ into the model and comparing the output logits against $v$.
For example, if a model has been trained on the database in Figure \ref{fig:db_countries}, we would evaluate it by getting it to complete sentences of the form `Astrid born\_in', and checking if the predicted token is `Singapore'.

Consider what this means in terms of $D$: rather than evaluating the model's predictions on the \emph{whole} of $D$, we are only evaluating against the fibers $D_{kq:}$ for $(k,q) \in \D$.
We do \emph{not} get the model's predictions on `nonsensical' sentences of the form `Malaysia born\_in', although we  should if we wish to evaluate the model's ability to reproduce the whole of $D$.
We ought to check that the model predicts none of the tokens to a high degree of certainty i.e. that $D_{kq:}$ is close to $0$ for $(k,q) \notin \D$.
But this is rarely done, for the obvious practical reason that there are simply too many `nonsensical' sentences to evaluate the model on.

In this paper, we will continue to adopt the conventional practice of evaluating only on $(k,q) \in \D$.
But we keep in mind that this could significantly reduce the rank required to memorize a database, as the following example shows:

\begin{example}
\label{ex:lowrank_db}
	Consider a database with the following `facts':
	$$
		\begin{array}{ccc}
			a & f & v \\
			b & g & v \\
			c & h & v \\
			d & i & v \\
			e & j & v
		\end{array}
	$$
	The tensor associated to this database has a single $v$-slice given by the following matrix with rank $5$:
	$$
D_{::v} = 
\begin{array}{cc}
    &
    \begin{array}{ccccc}
    f & g & h & i & j 
    \end{array}
    \\
    \begin{array}{c}
    a \\ b \\ c \\ d \\ e
    \end{array}
    &
    \left(
    \begin{array}{ccccc}
    1 & 0 & 0 & 0 & 0
    \\
    0 & 1 & 0 & 0 & 0
    \\
    0 & 0 & 1 & 0 & 0
    \\
    0 & 0 & 0 & 1 & 0
    \\
    0 & 0 & 0 & 0 & 1
    \end{array}
    \right)
\end{array}
	$$
	However, this database is very easy to `memorize': simply return $v$ regardless of which pairs of input are provided.
	This corresponds to setting $D_{::v}$ to be the all-ones matrix, which has rank 1.
\end{example}

Given a database $\D$ and tensor $D$ as defined in (\ref{eq:def_D}), say that another tensor $E$ \emph{represents} $\D$ as long as $E_{kq:} = D_{kq:}$ for $(k,q) \in \D$.
We may define the \emph{effective rank} of $\D$ to be the lowest rank of tensors $E$ that represent $D$.
As the preceding example shows, the effective rank of $\D$ (1 in the example) can be much less than the actual rank of $\D$ (5 in the example).

We do not have good theoretical bounds for the effective rank of $\D$, and we will not mention effective rank in the rest of the paper.
We simply note what this digression tells us: that models trained on databases are actually seeking to reproduce low rank representations of $D$, rather than $D$ itself.
As a consequence, we will expect the bounds in Proposition \ref{prop:db_rank} to significantly over-estimate the model rank required to memorize a database.

\begin{remark}
	Although we do not go further into this issue, it is clear that it has implications for LLM hallucinations when fed nonsensical inputs (e.g. `Where is Malaysia born?').
	One way to address this issue is to expand our model evaluation methods to include random selections of nonsensical inputs, and perhaps also including special tokens to allow the model to indicate that it does not know the answer to such questions.
\end{remark}


\section{Attention layers and their rank}
\label{sec:attention_layers_and_their_rank}

We now turn to the toy transformer models that we will be using: single-layer attention-only decoders, or simply \emph{attention layers}.
Following \cite{elhage2021mathematical,nichani2024understanding}, our models will be simplified versions of the original transformer architecture \cite{vaswani2017attention} without layer normalization, biases or positional encodings.  

An attention layer $\L$ is determined by the following parameters and weights:
\begin{itemize}
	\item $|\T|$, the number of tokens (commonly denoted $d_{vocab}$)
	\item $n_{heads}$, the number of attention heads
	\item $d_{model}$, the embedding dimension
	\item $d_{head,qk}$, the dimension of the query-key matrices
	\item $d_{head,ov}$, the dimension of the value-output matrices
	\item $W_E \in \RR^{|\T| \times d_{model}}$ and $W_U \in \RR^{d_{model} \times |\T|}$, the embedding and unembedding matrices
	\item $W^h_Q \in \RR^{d_{model} \times d_{head,qk}}$ and $W^h_K \in \RR^{d_{head, qk} \times d_{model}}$, the query and key matrices for each head $h \leq n_{heads}$
	\item $W^h_V \in \RR^{d_{model} \times d_{head,vo}}$ and $W^h_O \in \RR^{d_{head, vo} \times d_{model}}$, the value and output matrices for each head $h \leq n_{heads}$
\end{itemize}
where $d_{head,qk}, d_{head,vo} \leq d_{model} \leq |\T|$. 
We will \emph{not} assume that $d_{head,qk}$ and $d_{head,vo}$ are equal. 
As we shall see, model capacity depends more on $d_{head,vo}$ than on $d_{head,qk}$.

Instead of considering these matrices separately, we form the following $|\T| \times |\T|$ matrices (called \emph{circuits} in \cite{elhage2021mathematical}) whose ranks are bounded by their inner dimensions:
\begin{align*}
	W_{EU} &= W_E W_U,  & \rank(W_{EU}) &\leq d_{model}, \\
	W^h_{QK} &= W_E W^h_Q W^h_K W_E^\top, & \rank(W^h_{QK}) &\leq d_{head, qk},\\
	W^h_{VO} &= W_E W^h_V W^h_O W_U,  & \rank(W^h_{VO}) &\leq d_{head, vo}.
\end{align*}

Given a sequence of tokens $\mathbf{x} = (x_1, x_2, \dots, x_n) \in \T^n$, we may associate a matrix $X \in \RR^{n \times |\T|}$,
$$
	X = \left(\begin{matrix}
		\text{--- } \mathbf{e}_{x_1} \text{---} \\
		\text{--- } \mathbf{e}_{x_2} \text{---} \\
		\vdots \\
		\text{--- } \mathbf{e}_{x_n} \text{---} \\
	\end{matrix}\right)
$$
where $\mathbf{e}_{x_i}$ has $1$ in the $x_i^{th}$ position and $0$ everywhere else.
Given $X$, the model produces an $n \times |\T|$ matrix of logits:
\begin{equation}
\label{eq:logits_attn}
Z(\mathbf{x}) = X W_{EU} + \sum_{h = 1}^{n_{heads}} \mask(X W^h_{QK} X^\top) X  W^h_{VO},
\end{equation}
where $\mask$ is auto-regressively masked row-wise softmax, given by:
\begin{equation}
\mask(M)_{ij} = \begin{cases}
	\frac{\exp M_{ij}}{\sum_{k = 1}^i \exp M_{ik} } \mbox{ if } i \geq j, \mbox{ and} \\
	0 \mbox{ otherwise.}
\end{cases}
\end{equation}
We can then either apply row-wise argmax to get the next token $\argmax(Z(\mathbf{x})_{n:})$, or apply (non-masked) row-wise softmax to get a probability distribution $\softmax(Z(\mathbf{x})_{n:})$ over the next token.

\subsection{The rank of an attention layer}
We wish to compute the rank of an attention layer so that we can compare it to the rank of a database.
Our hypothesis is that an attention layer can only fully memorize or encode the data in a database if the layer rank exceeds the database rank.
We will see that this hypothesis is only partly true.

So far, our discussion about attention layers applies to the general setting where the length of input ``sentences'' $(x_1, \dots, x_n)$ may vary from sentence to sentence.
From now on, we will specialize to sentences of length at most $3$.

In this situation, we may define two 3-tensors per attention head.
The first tensor $V^h \in \RR^{|\K \cup \Q| \times |\Q| \times |\V|}$ represents the \underline{v}alue-output circuit, and has entries
\begin{equation}
	V^h_{tqv} = (W^h_{VO})_{tv}
\end{equation}
$V^h$ can be visualized as a loaf of sliced bread, with each $q$-slice equal to the same $|\K \cup \Q| \times |\V|$ submatrix of $W^h_{VO}$. 

The second tensor $A^h \in \RR^{|\K| \times |\Q| \times |\K \cup \Q|}$ is derived from the query-key or \underline{a}ttention circuit, and has entries
\begin{equation}
A^h_{kqt} = \begin{cases}
	\frac{\exp (W^h_{QK})_{qt}} {\sum_{t' = q,k} \exp (W^h_{QK})_{qt'}}  \mbox{ if } (k,q) \in \D, t \in \{k,q\}\\ 
	0 \mbox{ otherwise.}
\end{cases}
\end{equation}
$A^h$ contains the values of the attention matrix for all possible pairs $(k,q) \in \D$.
We have $(A^h_{kqk}, A^h_{kqq}) = \softmax((W^h_{QK})_{qk}, (W^h_{QK})_{qq}))$, which is the last row of $\mask(XW^h_{QK}X^\top)$ when $\mathbf{x} = (k,q)$.

Multiplying $A^h$ with $V^h$, we get a $|\K| \times |\Q| \times |\V|$ tensor whose $q^{th}$ slice is the matrix product $A^h_{:q:} V^h_{:q:}$, with entries
\begin{equation}
(A^h V^h)_{kqv} = \sum_{t} {A^h_{kqt} V^h_{tqv}}.
\end{equation}

Finally, we define the tensor $E \in \RR^{|\K| \times |\Q| \times |\V|}$ representing the \underline{e}mbed-unembed circuit:
\begin{equation}
E_{kqv} = (W_{EU})_{qv}.
\end{equation}
This is another loaf with each $k$-slice equal to the same $|\Q| \times |\V|$ submatrix of $W_{EU}$.

Putting these together, we get a tensor
\begin{equation}
\label{eq:layer_tensor}
L = E + \sum_{h = 1}^{n_{heads}} A^h V^h,
\end{equation}
and by construction of $L$, we have the following:
\begin{proposition}
The fibers $L_{kq:}$ are the logits when the sequence $(k,q)$ is fed into the model:
$$
	L_{kq:} = Z(k,q)_{2:}
$$
i.e. the final logits of $Z$ from (\ref{eq:logits_attn}) evaluated on $\mathbf{x} = (k,q)$.
\end{proposition}

\begin{remark}
This proposition and the construction of $L$ show that $W_{EU}$ (manifested in $E$) and $W^h_{VO}$ (manifested in $V^h$) act as linear operations when computing the output logits.
By contrast, $W^h_{QK}$ acts more like a look-up table for populating $A^h$.
The construction of $L$, and specifically of $A^h$ and $V^h$ and their composition, makes explicit the observation in \cite{elhage2021mathematical} that `logits are a linear function of tokens'.
It also sheds light on the `additive motif' of \cite{chughtai2024summing} by showing how the values in $W^h_{VO}$ add up to approximate $D$. We illustrate with an example:
\end{remark}

\begin{example}
	We show how the circuits in an attention layer might be defined so as to give the slices $D_{:\beta:}$ and $D_{:\lambda:}$ from Example \ref{ex:db}.
	Consider a single head with the following circuits:
\begin{align*}	
W^1_{QK} &= \begin{array}{cc}
    &
    \begin{array}{ccccc}
    a & b & c & \beta & \lambda
    \end{array}
    \\
    \begin{array}{c}
    \beta \\ \lambda 
    \end{array}
    &
    \left(
    \begin{array}{ccccc}
    1 & 1 & 1 & 1 & 1
    \\
    1 & 1 & 1 & 1 & 1
    \end{array}
    \right)
\end{array}
\\
W^1_{VO} &= \begin{array}{cc}
    &
    \begin{array}{cc}
    m & s 
    \end{array}
    \\
    \begin{array}{c}
    a \\ b \\ c \\ \beta \\ \lambda
    \end{array}
    &
    \left(
    \begin{array}{cc}
    0 & 0 
    \\
    0 & 4
    \\
    4 & 0 
    \\ 
    0 & 2
    \\
    2 & 0 
    \end{array}
    \right)
\end{array}.
\end{align*}
We have $V^1_{:\beta:} = V^1_{:\lambda:} = W^1_{VO}$, while the slices of $A^1$ are
\begin{align*}
A^1_{:\beta:} &= \begin{array}{cc}
    &
    \begin{array}{ccccc}
    a & b & c & \beta & \lambda
    \end{array}
    \\
    \begin{array}{c}
    a \\ b \\ c
    \end{array}
    &
    \left(
    \begin{array}{ccccc}
    \frac{1}{2} & 0 & 0 & \frac{1}{2} & 0
    \\
    0 & \frac{1}{2} & 0 & \frac{1}{2} & 0
    \\
    0 & 0 & \frac{1}{2} & \frac{1}{2} & 0
    \end{array}
    \right)
\end{array}
\\
A^1_{:\lambda:} &= \begin{array}{cc}
    &
    \begin{array}{ccccc}
    a & b & c & \beta & \lambda
    \end{array}
    \\
    \begin{array}{c}
    a \\ b \\ c
    \end{array}
    &
    \left(
    \begin{array}{ccccc}
    \frac{1}{2} & 0 & 0 & 0 & \frac{1}{2} 
    \\
    0 & \frac{1}{2} & 0 & 0 & \frac{1}{2} 
    \\
    0 & 0 & \frac{1}{2} & 0 & \frac{1}{2} 
    \end{array}
    \right)
\end{array}.	
\end{align*}
The slices of the product $A^1V^1$ are:
\begin{align*}	
(A^1V^1)_{:\beta:} &= \begin{array}{cc}
    &
    \begin{array}{cc}
    m & s 
    \end{array}
    \\
    \begin{array}{c}
    a \\ b \\ c
    \end{array}
    &
    \left(
    \begin{array}{cc}
    0 & 1 
    \\
    0 & 3 
    \\
    2 & 1 
    \end{array}
    \right)
    \end{array}
\\
(A^1V^1)_{:\lambda:} &= \begin{array}{cc}
    &
    \begin{array}{cc}
    m & s 
    \end{array}
    \\
    \begin{array}{c}
    a \\ b \\ c
    \end{array}
    &
    \left(
    \begin{array}{cc}
    1 & 0
    \\
    1 & 2 
    \\
    3 & 0 
    \end{array}
    \right).
\end{array}
\end{align*}
Applying $\argmax$ to these slices gives the slices $D_{:\beta:}$ and $D_{:\lambda:}$. We have not needed the $W_{EU}$ circuit here, but an actual model is likely to offload some values to this circuit.

This example shows that a single value-output circuit of rank $2$ can account for \emph{both} slices $D_{:\beta:}$ and $D_{:\lambda:}$ even though they have a combined rank of $3$.
This is partly achieved by exploiting the shared values of $\beta$ and $\lambda$ on $b$ and $c$ (as alluded to in Example \ref{ex:db}), but also from the fact that $\rank(A^1 V^1) = 3$ even though $\rank(V^1) = 2$.
Indeed, $\rank(AB) \not \leq \min(\rank(A), \rank(B))$ for $3$-tensors.

Note also that $\rank(A^1) = 4$ even though $\rank(W^1_{QK}) = 1$, showing that the rank of $W^1_{QK}$ can be low without compromising the rank of $A^1$.
(In fact, we would have the same $A^1$ even if all the entries of $W^1_{QK}$ were zero.)
\end{example}

\begin{remark}
On one hand, this example shows that direct encoding of facts is possible at the level of circuits, and 
suggests that interpreting layers might be easier at the circuit level, rather than at the level of weights or embeddings.
On the other hand, even though the weights in this example were chosen to be simple, we are already seeing that $W^1_{VO}$ need not look at all like either $D_{:\beta:}$ or $D_{:\lambda:}$.
If $\rank(L)$ exceeds $\rank(D)$, there will be infinitely many ways for the circuits to be defined to give $D$, and it is reasonable to suppose that many of these circuits will be very hard to interpret, with the values of one slice of $D$ spread over multiple heads and over the embedding circuit as well.
\end{remark}

We define the \emph{rank} of a model $\L$ to be the rank of $L$ from (\ref{eq:layer_tensor}).
Based on the preceding discussion and example, we expect the rank of a model to depend more on the ranks of $W_{EU}$ and $W^h_{VO}$, but not so much on $W^h_{QK}$.
We have the following upper bound for the rank of $L$ as defined in (\ref{eq:layer_tensor}):
\begin{proposition}
\label{prop:layer_rank_upper}
	$$\rank(L) \leq d_{model} + n_{heads} \cdot d_{head,vo} \cdot |\Q|.$$
\end{proposition}
\begin{proof}
	Along each $q$-slice, we have $\rank((A^hV^h)_{:q:}) \leq \rank(V^h_{:q:}) \leq \rank(W^h_{VO}) \leq d_{head, vo}$.
	As there are $|\Q|$ such slices, we have $\rank(A^hV_h) \leq d_{head,vo} \cdot |\Q|$.
	Meanwhile, $\rank(E) \leq \rank(W_{EU}) \leq d_{model}$.
	Finally, applying the subadditivity of rank yields the result.
\end{proof}

Surprisingly, the inequality contains $|\Q|$, which is a property of $\D$ rather than $\L$, suggesting that the capacity of an attention layer might depend on the context in which it is being trained and evaluated.
To get a tighter upper bound, we should replace $|\Q|$ with the number of $q$s that each head contributes to, which could vary from head to head.
On the other hand, to get a \emph{lower} bound, we feel it is reasonable to assume that each head contributes to at least \emph{one} $q$. 
We also feel it is reasonable to assume that the process of training an attention layer will cause the spans of $E$ and $V^h$ to have as little intersection as possible, so that we are likely to have 
\begin{equation}
\label{eq:layer_rank_lower}
	d_{model} + n_{heads} \cdot d_{head,vo} \leq \rank(L).
\end{equation}

Note that $d_{head,qk}$ does not appear in Proposition \ref{prop:layer_rank_upper} or (\ref{eq:layer_rank_lower}).
As the example shows, $\rank(A^h)$ can be much higher than $\rank(W^h_{QK})$, so $d_{head,qk}$ does not significantly affect the capacity of the attention layer.
Indeed, the $W^h_{VO}$ weights bear the burden of storing the entries of $D$, while the $W^h_{QK}$ weights are only responsible for indicating which key-query pairs in $\D$ are relevant for each head.

Now that we have defined the rank of a database and the rank of an attention layer, we might guess that a test 
for whether an attention layer can memorize a database would be
$
	\rank(D) \leq \rank(L).
$
Combining Proposition \ref{prop:db_rank} and (\ref{eq:layer_rank_lower}) gives the following condition:
\begin{equation}
\label{eq:condition}
	\min\left( \sum_{k \in \K} |\V_k|, \sum_{q \in \Q} |\V_q| \right) \leq d_{model} + n_{heads} \cdot d_{head,vo}. 
\end{equation}

While this is a sufficient condition, it assumes that $L$ exactly replicates $D$ i.e. $L = D$, which is overly strict.
Indeed, we usually either take $\argmax(L_{kq:})$ if we want to get a token prediction, or $\softmax(L_{kq:})$ if we want to get a probability distribution over output tokens.
A more appropriate condition would be
$\rank(D) \leq \rank(\argmax(L))$ or $\rank(D) \leq \rank(\softmax(L))$.
One would hope that $\rank(\argmax(L))$ and $\rank(\softmax(L))$ are related to $\rank(L)$ in some way.
Unfortunately, as we shall see, $\argmax$ and $\softmax$ can depend very little on $\rank(L)$.

\subsection{Effects of argmax and softmax on rank}
Recall the definitions of row-wise $\argmax$ and $\softmax$:
\begin{align}
	\argmax(M)_{ij} &= \begin{cases}
		1 \mbox{ if } M_{ij} = \max(M_{i:}) \\
		0 \mbox{ otherwise}
	\end{cases} \label{eq:argmax} \\
	\softmax(M)_{ij} &= \frac{\exp(M_{ij})}{\sum_k \exp(M_{ik})}. \label{eq:softmax}
\end{align}
If $\rank(M) = 1$, it is easy to see that $\rank(\argmax(M)) \leq 2$.
But if $\rank(M) \geq 2$, anything goes:
\begin{proposition}
	\label{prop:argmax_rank}
	For $2 \leq r \leq n$, there are $n \times n$ matrices $M$ with  $\rank(M) = r$ but $\rank(\argmax(M)) = n.$
\end{proposition}
\begin{proof}
	Take $n$ distinct points $\mathbf{v}_1, \dots, \mathbf{v}_n$ on the unit sphere in $\RR^r$, and let $M$ be their Gram matrix: $M_{ij} = \langle \mathbf{v}_i, \mathbf{v}_j \rangle$.
	Then $M$ is a rank $r$ matrix with ones along the diagonal and off-diagonal entries strictly less than $1$, so $\argmax(M)$ is the $n \times n$ identity matrix.
\end{proof}

With an extra step, we can get a similar result for $\softmax$:
\begin{proposition}
	\label{prop:softmax_rank}
	For $2 \leq r \leq n$, there are $n \times n$ matrices $M$ with  $\rank(M) = r$ but $\rank(\softmax(M)) = n.$
\end{proposition}
\begin{proof}
	Let $M$ be the matrix from the previous proof. 
	Simply choose a large enough $c > 0$ such that $(e^{M_{ii}})^c > \sum_{j\neq i} (e^{M_{ij}})^c$ for all $i$.
	This is always possible as long as the diagonal is strictly greater than the off-diagonal entries, as its exponent will eventually grow large enough to dominate the other entries.
	Then $\rank(cM) = \rank(M) = r$, but $\softmax(cM)$ is diagonally dominant, hence invertible.
\end{proof}

Both of these results are somewhat artificial, as they assume that we can compute with arbitrarily small or arbitrarily large magnitudes.
Having limited machine precision would reduce the maximum possible rank of $\argmax(M)$, as it would prevent us from distinguishing points that are too close.
Similarly, the proof for $\softmax$ requires that we can compute with large enough $c$.
It seems unlikely that we will encounter such extreme examples when dealing with logits produced by attention layers.
Indeed, the fact that a `softmax bottleneck' has been observed (e.g. in \cite{yang2017breaking}) suggests that softmax does not significantly increase the rank of matrices encountered in real-world scenarios.
Nevertheless, these results act as a warning that $\argmax$ and $\softmax$ can potentially distort rank.







In order to regain some control on rank, we will define a modified version of softmax.
Beyond rank considerations, this modification is motivated by a separate consideration:

Suppose we have a vector of next-token probabilities $\mathbf{p} = (p_1, \dots, p_{|\T|})$.
Such a vector might be produced by applying $\softmax$ to the logits $L_{kq:}$, for example.
If the correct token is $t$, $p_t$ represents the probability that we get the correct token if we draw from the distribution $\mathbf{p}$, which might still be very small even if it is larger than all the other probabilities.
It would be too lenient to say that a prediction is correct as long as $p_t$ is the largest.

A better way would be to additionally require that $p_t$ exceed a certain threshold $\tau$, e.g. $\tau = 0.5$ or $0.9$.
If $\tau \geq 0.5$, the condition that $p_t \geq \tau$ automatically implies that $p_t$ is the largest in $\mathbf{p}$.
With this in mind, for $\tau \in [0.5, 1]$, we define
\begin{equation}
	\softmax_{\geq \tau}(M)_{ij} = \begin{cases} 
	1 \mbox{ if } \softmax(M)_{ij} \geq \tau, \\
	0 \mbox{ otherwise.}
	\end{cases}
\end{equation}
We also define the \emph{$\tau$-accuracy} to be
\begin{equation}
	\acc_{\geq \tau}(L, D) = \frac{1}{|\D|} \sum_{(k, q, v) \in \D} \softmax_{\geq \tau}(L)_{kqv},
\end{equation}
This function simply counts the number of correct predictions, divided by the total number of triples in $\D$.
We say that a layer \emph{$\tau$-memorizes} a database if $\acc_{\geq \tau}(L,D) = 1$. 
In this case, we have $D = \softmax_{\geq \tau} (L)$.
So we expect that a model $\L$ can $\tau$-memorize $\D$ if 
\begin{equation}
\rank(D) \leq \rank(\softmax_{\geq \tau}(L)).
\end{equation}
One would hope that $\rank(\softmax_{\geq \tau}(L))$ now bears some relation to $\rank(L)$.
Unfortunately, by choosing large enough $c$ in the proof of Proposition \ref{prop:softmax_rank}, we can theoretically also get $\rank(\softmax_{\geq \tau})$ to take any value.\footnote{For example, the $c$ for Proposition \ref{prop:softmax_rank} also works for $\tau = 0.5$.}
Nevertheless, we can hope that choosing a large enough $\tau$ will cause the $c$ required to be unreasonably big, so that in practice $\softmax_{\geq \tau}$ will not increase the rank by too much.

We will see empirically that varying $\tau$ allows us to control how far the rank of $\softmax_{\geq \tau} (L)$ deviates from the rank of $L$.
As $\tau$ increases, $\rank(\softmax_{\geq \tau}(L))$ decreases and gets closer to $\rank(L)$, so that (\ref{eq:condition}) becomes a good indicator of whether $\L$ can memorize $\D$.
However, when $\tau$ is too big, $\rank(\softmax_{\geq \tau} (L))$ can become less than $\rank(L)$, and 
eventually goes to $0$ when $\tau = 1$ (since a softmax probability of $1$ requires the other logits to be $-\infty$).

\section{Experimental results}
\label{sec:experimental_results}
We have so far focused on defining and understanding the theoretical properties of the ranks of databases and attention layers.
We turn now to experimental results to shed further light on these quantities.
We make use of a small dataset of randomly generated databases and attention layers trained on these databases.
Our dataset consists of 548 databases with at most 200 triples, and 364 attention layers with $1 \leq d_{head, ov}, d_{head,kq} \leq d_{model} \leq 6$ and $1 \leq n_{heads} \leq 4$.

\begin{figure}[ht]
\centering     
\subfigure[Number of database triples against the database rank upper bound from Proposition \ref{prop:db_rank}.]{\label{fig:triples_db_rank}\includegraphics[width=0.48\columnwidth]{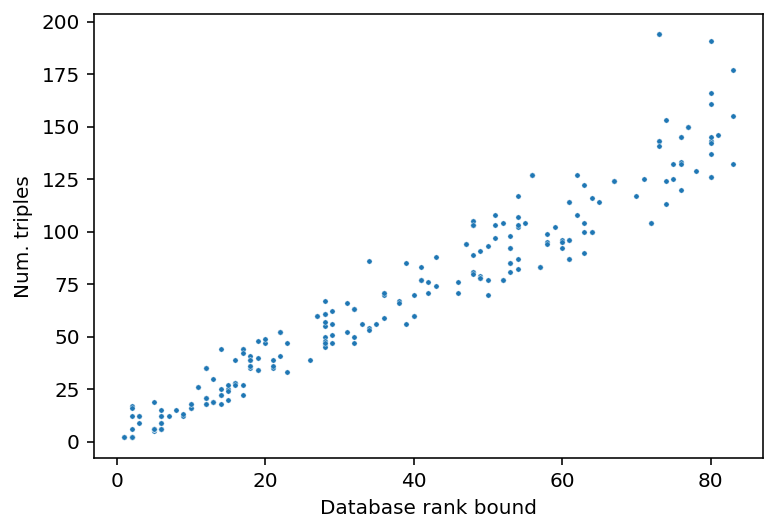}}
\hspace{2mm}
\subfigure[Number of non-embedding attention layer parameters against the attention layer rank lower bound estimate from (\ref{eq:layer_rank_lower}).]{\label{fig:param_layer_rank}\includegraphics[width=0.48\columnwidth]{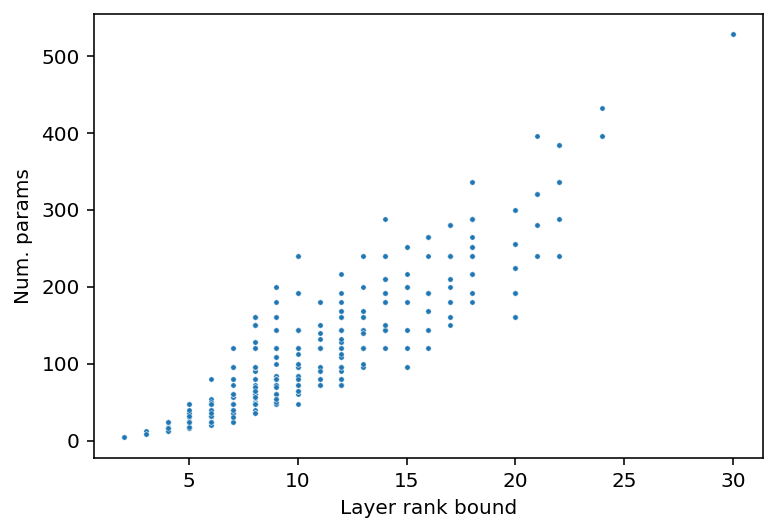}}
\caption{Relationship of database rank and attention layer ranks to other measures of database and layer size.}
\end{figure}

\begin{figure}[ht]
\centering     
\subfigure[Average accuracy with $\argmax$]{\label{fig:acc_0}\includegraphics[width=0.48\columnwidth]{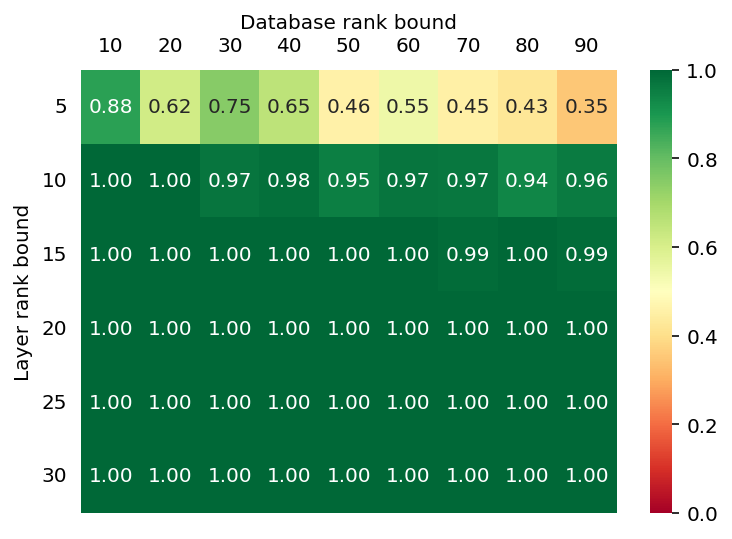}}
\hspace{2mm}
\subfigure[Average accuracy with $\softmax_{\geq 0.75}$]{\label{fig:acc_75}\includegraphics[width=0.48\columnwidth]{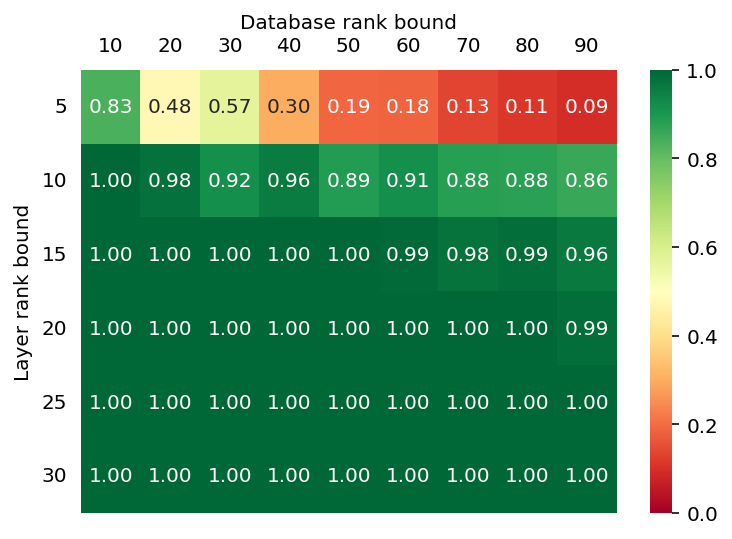}}
\subfigure[Average accuracy with $\softmax_{\geq 0.95}$]{\label{fig:acc_95}\includegraphics[width=0.48\columnwidth]{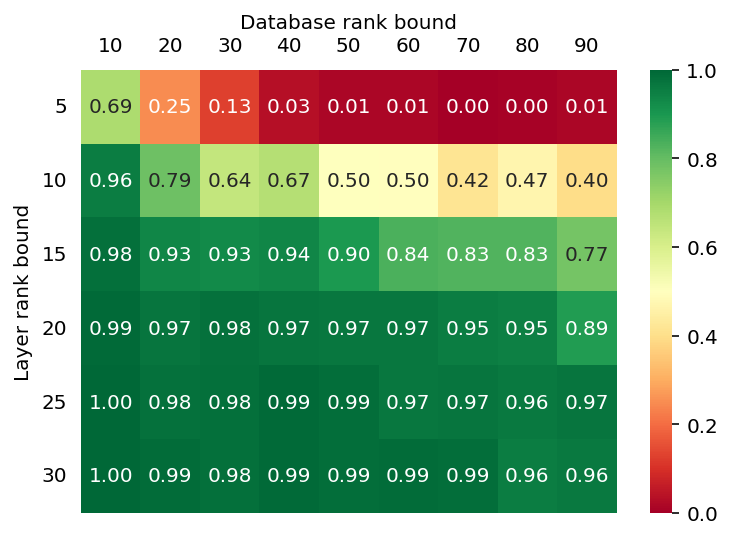}}
\hspace{2mm}
\subfigure[Average accuracy with $\softmax_{\geq 0.99}$]{\label{fig:acc_99}\includegraphics[width=0.48\columnwidth]{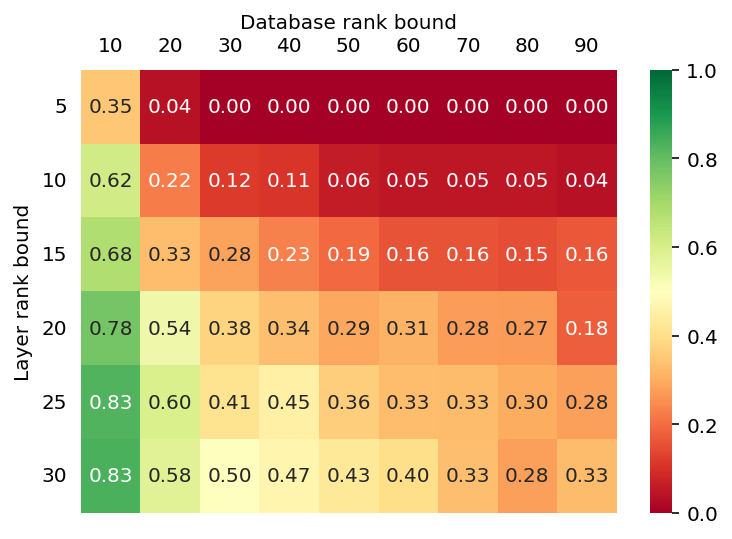}}
\caption{Average accuracy, with predictions made using $\argmax$ or $\softmax_{\geq \tau}$, $\tau = 0.75, 0.95, 0.99$, of attention layers trained on databases.
Layer rank bound is the lower bound from (\ref{eq:layer_rank_lower}), while database rank bound is the upper bound from Proposition \ref{prop:db_rank}.
Each $(i,j)^{th}$ cell averages accuracies of layer-database pairs with layer rank bound $\in (i-5,i]$ and database rank bound $\in (j-10, j]$.
}
\end{figure}

We begin by showing the relationship of the rank bounds in Proposition \ref{prop:db_rank} and (\ref{eq:layer_rank_lower}) to other popular measures of database size and attention layer size.
Figure \ref{fig:triples_db_rank} shows that the number of database triples varies approximately linearly with the upper bound of database rank from Proposition \ref{prop:db_rank}, with the number of triples being approximately 1.8 times the database rank upper bound. This suggests that both $\rank(\D)$ and $|\D|$ can be used as measures of database size, after the appropriate scaling factor has been accounted for.

On the other hand, Figure \ref{fig:param_layer_rank} shows that the number of (non-embedding) parameters and the lower bound from  (\ref{eq:layer_rank_lower}) are still positively correlated, but non-linearly and with wider spread, as is to be expected by comparing (\ref{eq:layer_rank_lower}) with
$$
	\mbox{Num. params} = 2 \cdot n_{heads} \cdot d_{model} \cdot (d_{head,vo} + d_{head,qk}).
$$

From this dataset, 3,947 random pairs of databases and attention layers were selected, and the layers were trained on database triples for a maximum of 2,000 epochs. The final (training) accuracy at various thresholds was recorded.\footnote{We are interested in evaluating how well attention layers have memorized the training set, so there is no separate test set.  We evaluate accuracy based only on final token predictions, but the model does not `know' this, and seeks to predict earlier tokens too.} Figure \ref{fig:acc_0} shows how well these attention layers have succeeded in memorizing the databases, assuming that predictions are taken using argmax of logits.
We see that even tiny models of rank up to 10 can memorize databases of rank 80, providing evidence of the significant rank-increasing effect of $\argmax$ that we might expect based on Proposition \ref{prop:argmax_rank}.

Figures \ref{fig:acc_75} to \ref{fig:acc_99} show how the memorization success changes when we use $\softmax_{\geq \tau}$ with varying values of $\tau$.
When $\tau = 0.5$, the results were almost identical to Figure \ref{fig:acc_0} (and are thus omitted), showing that $\softmax_{\geq 0.5}$ can distort rank as much as $\argmax$.
When $\tau = 0.75$ and $0.95$, we see a stronger dependence between layer rank and database rank. 
At $\tau = 0.75$, a layer can memorize a database whose rank is about 3-4 times the layer rank, while at $\tau = 0.95$, this factor drops to 1-2.
However, when $\tau = 0.99$, attention layers cannot even memorize databases with rank less than $\rank(L)$, suggesting that $\tau = 0.99$ is too strict of a requirement (or more epochs are needed).

Across all values of $\tau$, increasing the lower bound from (\ref{eq:layer_rank_lower}) increases the memorization accuracy, showing that despite its failure to account for the effects of $\softmax$, the bound is still a meaningful measure of attention layer capacity.

\begin{figure}[ht]
\vskip 0.2in
\begin{center}
\centerline{\includegraphics[width=0.6\columnwidth]{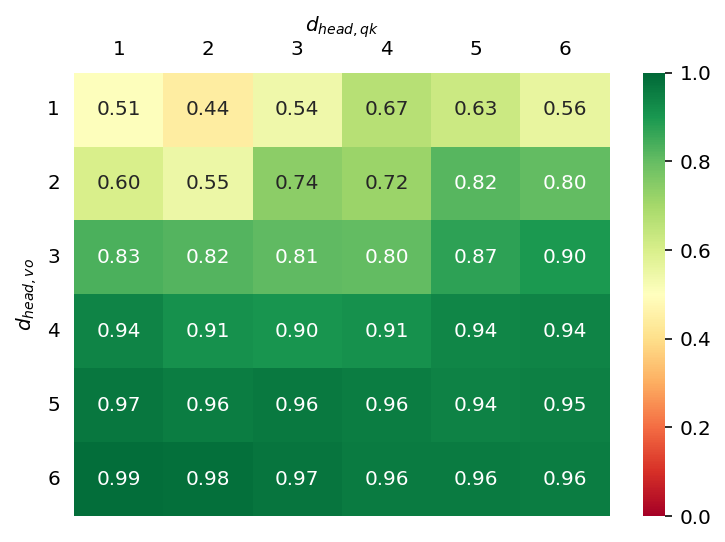}}
\caption{Average accuracy of attention layers with $\softmax_{\geq 0.95}$ as $d_{head,vo}$ and $d_{head,qk}$ vary.}
\label{fig:vo_qk}
\end{center}
\vskip -0.2in
\end{figure}

Finally, Figure \ref{fig:vo_qk} shows the effect of varying $d_{head,vo}$ and $d_{head,qk}$ on accuracy with $\softmax_{\geq 0.95}$, and  corroborates our observation following (\ref{eq:layer_rank_lower}) that $d_{head,vo}$ ought to play a bigger role than $d_{head, qk}$ in determining the capacity of an attention layer.
We see that while increasing $d_{head, vo}$ and $d_{head,qk}$ both generally increase the average accuracy, the effect is more pronounced with $d_{head,vo}$.
Indeed, accuracy increases from top right to bottom left along all diagonals of constant $d_{head,vo} + d_{head,qk}$, suggesting that we can increase capacity without more parameters by raising $d_{head,vo}$ at the expense of $d_{head,qk}$.

\section{Contributions, Limitations \& Future Work}
\label{sec:contributions_limitations_future_work}

\subsection{Contributions}
In this paper, we have attempted to understand fact storage in attention layers from a linear-algebraic perspective.
We have introduced novel $3$-tensors $D$ and $L$ corresponding to databases and attention layers, and proved bounds for their rank, while also providing empirical evidence of the usefulness of these quantities as well as the impact of $\argmax$ and $\softmax$ on rank.

Our main theoretical contribution is the construction of $L$, which provides a deeper understanding of the roles of the query-key and value-output weights, and (\ref{eq:layer_rank_lower}), which provides a measure of the capacity of an attention layer. 
By decoupling $d_{head,vo}$ from $d_{head,qk}$, we showed that we can increase model capacity without increasing model size by increasing $d_{head,vo}$ at the expense of $d_{head,qk}$, a result that we have observed empirically.
We have shown by example that it is possible to use the components of $L$ to guide the encoding of facts directly into circuits, but also that even with hand-crafted examples, it remains hard to localize where specific facts are stored.

\subsection{Limitations and Extensions}

Our paper suggests many extensions and directions for future work.
Although we focus on the case of database triples, it is not hard to generalize $D$ to an $n$-tensor storing $n$-gram statistics, while $L$ will still have the same form as (\ref{eq:layer_tensor}), but with more involved $A^h$ and $V^h$. Of course, our rank bounds will not hold in this more general setting, so new bounds will have to be derived.
Our remarks in Section \ref{sec:low_rank_representations_of_databases} concerning low-rank representations of $D$ still hold, and future work could look into deriving better bounds on the \emph{effective} rank of $D$.
General efforts to study the effect of $\softmax$ on rank in more realistic settings, and to derive tighter bounds on the tensor rank will be also useful.

We expect the addition of layer norms and biases to only minimally affect the definition of the rank of an attention layer, but the effect of positional embeddings, MLP layers and having multiple layers is unknown.
Even in our original setting, we note in Figure \ref{fig:vo_qk} that for small values of $d_{head,vo}$, increasing $d_{head,qk}$ causes an increases in accuracy that is not accounted for by (\ref{eq:layer_rank_lower}), suggesting that more investigation needs to be done to derive better bounds on $\rank(L)$ and $\rank(A^h)$ that involve $d_{head,qk}$.
Note also, however, that when $d_{head,vo}$ is larger, the positive effect of increasing $d_{head,qk}$ starts to vanish and even starts to become detrimental. More work is needed to confirm if this pattern persists in larger models, and why. 

On the experimental front, we have only investigated small toy models and databases. 
It remains to be seen if our results continue to hold for larger models trained on larger databases.
We anticipate that the rank-increasing effects of $\softmax$ will become more severe as the rank of $L$ increases, and it might be necessary to scale $\tau$ according to the rank of $L$, instead of keeping it constant across models as we have done.
As the models and databases scale up, it will also quickly become infeasible to construct and study the tensors $D$ and $L$ holistically. Future work might thus need to consider ways of analysing sub-tensors of $D$ and $L$ and their rank.

In addition to considering much larger models and databases, future work could make the databases more realistic by adding noise and missing data, by using real-world databases or RDF triple stores, and by expressing facts in natural language instead of formal RDF triples.
The impact of varying $d_{head,vo}$ and $d_{head,qk}$ separately could also be tested on language models that are not solely trained for database memorization, to see if reducing $d_{head,qk}$ has detrimental effects on other tasks.

Evidently, our work is just a small step, and suggests more questions than it answers. Nevertheless, we hope it demonstrates the usefulness and potential of a tensor-theoretic treatment of tasks and models, and motivates further work in this direction for other tasks and types of models.




\bibliographystyle{plain}
\bibliography{biblio}

\end{document}